\documentclass[conference]{IEEEtran}
\IEEEoverridecommandlockouts

\usepackage{cite}
\usepackage{amsmath,amssymb,amsfonts,amsthm}
\usepackage{algorithm}
\usepackage{algorithmic}
\usepackage{graphicx}
\usepackage{booktabs}
\usepackage{enumitem}
\usepackage{textcomp}
\usepackage{xcolor}

\newtheorem{theorem}{Theorem}

\newcommand{\AP}{\mathsf{AP}}
\newcommand{\Bad}{\mathsf{Bad}}
\newcommand{\Trace}{\mathsf{Tr}}

\begin{document}

\title{GARV: Graph-Aware Refinement and Safety Verification for High-Level Robot Plans}
\author{
\IEEEauthorblockN{
    Stabak Das\IEEEauthorrefmark{1},
    Priyesh Ranjan\IEEEauthorrefmark{1},
    Xiangfang Li,
    Lijun Qian
}
\IEEEauthorblockA{
    CREDIT Center \& Department of Electrical and Computer Engineering \\ Prairie View A\&M University, Texas A\&M University System, 
    Prairie View, TX 77446, USA \\
    Emails: \{sdas1;
    prranjan;
    xili;
    liqian\}@pvamu.edu
}
\thanks{\IEEEauthorrefmark{1}Equal contribution.}
}

\maketitle

\begin{abstract}
High-level robot actions such as \texttt{goto(region)} can hide intermediate traversal, allowing a monitor of the listed actions to accept a plan whose graph-level execution violates a safety rule. We present Graph-Aware Refinement and Verification (GARV), which composes graph-grounded action contracts with a safety monitor to produce a model-relative certificate or counterexample and, when route choice is enforceable, a minimum-cost safe refinement. In a 28-case controlled audit, refinement exposed all 12 targeted hidden hazards and produced the expected verdicts on 16 controls; a separate 14-case SPINE-RoboGuard instantiation yielded two false-safe cases and five action-interface errors. In standalone finite-graph simulation, all 16 designed cases met their expected outcomes, and product verification agreed with exhaustive enumeration on classification, minimum safe hop count, and witness safety in 160 comparisons across 40 graphs. Across 90 graph-contract evaluations on 30 graphs with a hazard deliberately placed on the selected shortest route, GARV found a safe refinement in 72 evaluations and no safe route within the permitted hop bound in 18. For the static forbidden-region rule, synthesis matched a simple hazard-deletion baseline in all 90 evaluations; GARV's guarantees remain conditional on the semantic graph, action contract, safety specification, and enforceability of any selected route.
\end{abstract}

\begin{IEEEkeywords}
Robot safety, semantic graphs, temporal logic, runtime verification, hierarchical planning
\end{IEEEkeywords}

\section{Introduction}
Robots increasingly receive high-level commands such as ``go to the hall'' or ``inspect the sensor.'' A planner may represent these commands as a short sequence of symbolic actions, while navigation and API layers resolve them into graph paths and implicit effects. A safety monitor reasons only about the events it receives. If the monitored trace omits intermediate regions or implicit navigation, its verdict may not apply to the behavior induced by the command.

Consider a robot asked to reach a safe hall. Its surface plan may contain only \texttt{goto(hall)}, although the selected graph route passes through a forbidden construction area $C$. Under $\varphi_C=\mathbf{G}\neg\texttt{goto}(C)$, the surface trace is accepted because it omits entry into $C$. There is a second difficulty even when the selected route is safe: another route permitted by the navigation policy may pass through $C$. Checking one selected expansion cannot certify every execution that the policy permits. Both cases require an explicit account of the graph-level events represented by a high-level action.

We formulate this account as a graph-grounded action contract. Given a semantic graph, a high-level plan, an action contract, and a safety monitor, GARV explores a graph-monitor scheme to check the routes admitted by the contract. It returns a model-relative certificate when all admitted traces satisfy the specification, or a concrete violating one when any one trace does not. When the navigation interface can enforce a chosen route, the same scheme supports minimum-cost safe refinement while preserving the plan's high-level action endpoints. The existence of a safe route alone does not certify execution if the interface chooses an unsafe one.

We use SPINE and RoboGuard~\cite{ravichandran2025spine,ravichandran2026roboguard} as one instantiation: SPINE supplies high-level actions, and RoboGuard supplies scene-grounded safety specifications for monitoring. The GARV algorithm does not require these systems to operate independently; its inputs are a graph, compatible action and event semantics, and a monitor for finite safety violations.

Our contributions are:
\begin{itemize}[leftmargin=*,topsep=2pt,itemsep=1pt]
    \item We define graph-grounded action contracts that expose intermediate navigation and documented implicit effects, including the distinction between one selected route and a set of permitted routes.
    \item We present a scheme-graph algorithm for route-set certification and counterexample generation,  with minimum-cost safe refinement when route choices can be enforced. We state the conditions under which these outputs are correct relative to the graph and contract.
    \item We evaluate the abstraction gap in 28 controlled and 14 SPINE-RoboGuard cases, then evaluate GARV's route-set and synthesis modes in 16 designed cases, 160 comparisons with exhaustive enumeration, and 90 graph-contract scaling evaluations.
\end{itemize}

Along with the results on the finite semantic graphs, we also test the additional monitor-contract behavior by using the route-set and checkpoint-dependent cases. We also inspect the execution in case of varied graph fidelity, the modeled action contract, the safety specification, and enforcement of any synthesized route.

The rest of the manuscript is organized as follows. Section~\ref{sec:related} discusses the relevant literature in the area of hierarchical planning and runtime assurance for foundation-model run agents. Section~\ref{sec:method} describes the problem statement and the proposed solution methodology along with the relevant graph-refinement algorithm. Section~\ref{sec:methodology} describes the experimental setup and defines the metrics used for ascertaining the obtained results. Section~\ref{sec:results} contains the obtained results for the controlled verification and end-to-end pipeline and tests with SPINE and RoboGuard along with discussion on the obtained results. Finally, the paper is concluded in Section~\ref{sec:conclusion}.
\section{Related Work}
\label{sec:related}

We review graph-grounded planning, temporal-logic refinement, and safety verification for language-enabled robots. The common concern is how a high-level plan relates to the actions and events produced during execution.

\subsection{Hierarchical Planning with Semantic Graphs}
Task-and-motion planning connects symbolic task choices to geometric feasibility and motion generation. Hierarchical methods commit to abstract choices and refine them into executable motions as planning proceeds~\cite{kaelbling2011hierarchical}. Language-enabled systems use semantic or 3D scene graphs to ground instructions and reduce the planning search space. SayPlan combines hierarchical scene graphs, classical path planning, and simulator feedback for long-horizon tasks~\cite{rana2023sayplan}; SPINE builds and updates semantic maps while refining missions with incomplete natural-language specifications~\cite{ravichandran2025spine}. Scene-graph planning has also combined graph attributes, LTL task constraints, and hierarchical search~\cite{dai2024optimal}.

VeriGraph uses scene graphs to represent manipulation states and checks proposed actions against graph-based constraints before executing them~\cite{ekpo2026verigraph}. These works establish scene graphs as useful models for planning and action validation. GARV focuses on a particular planning interface: the graph-level event traces that a high-level robot API action can produce, including intermediate navigation that may be absent from the proposed action sequence.

\subsection{Temporal Logic, Abstraction, and Refinement}
Temporal-logic robot planning often combines a motion model with an automaton for the specification and searches their product. Kress-Gazit et al.\ synthesize controllers from LTL task specifications using a discrete workspace abstraction and compositions of low-level controllers~\cite{kressgazit2009reactive}. Fu et al.\ plan under LTL constraints in probabilistic semantic maps~\cite{fu2016optimal}, while scene-graph planning has used LTL automata to guide hierarchical search~\cite{dai2024optimal}. GARV builds on these established product-graph and automata techniques.

Formal abstraction research also examines when high-level strategies preserve temporal properties under refinement. De Giacomo et al. study temporally lifted abstractions relating high- and low-level models~\cite{de2025situation}. GARV applies this concern at the robot API boundary: an action contract specifies the graph traces admitted by an action's route policy, and product exploration checks those traces while tracking progress through the original plan. The certificate is relative to that contract. Our planner--monitor audit evaluates selected routes; separate finite-graph experiments evaluate all-shortest and bounded-hop route sets, as well as safe-route synthesis.

\subsection{Safety Verification for Language-Enabled Robots}
Runtime-assurance architectures protect uncertified components using monitors and fallback behavior. SOTER, for example, switches between a high-performance controller and a verified safety controller~\cite{desai2019soter}. At the language-planning interface, the Safety Chip encodes prohibitions as temporal-logic constraints and filters unsafe actions~\cite{yang2024safetychip}. RoboGuard grounds contextual safety rules in a semantic graph and checks proposed robot API actions~\cite{ravichandran2026roboguard}. VerifyLLM translates instructions into LTL and uses LLM-guided analysis to check high-level task plans before execution~\cite{grigorev2025verifyllm}. Wang et al. connect LTL-based task-plan supervision with motion planning by passing hazard information to the motion planner~\cite{wang2024crosslayer}.

GARV makes the action-to-event relationship an explicit verification input. Given a graph, action contract, and safety monitor, it checks the admitted route set and returns a model-relative certificate or a violating route. When route choices can be enforced, it also searches for a minimum-cost safe refinement that preserves the high-level action sequence.

\section{Problem Formulation and Method}
\label{sec:method}

This section focuses on the formulating the input and the action space for the embodied agents. We also provide model-relative correctness for graph-level traces and instantiation for certifiable plans. We also provide an algorithmic definition for the GARV algorithm showcasing the outputs returned for different input cases.

\begin{figure*}[t]
    \centering
    \includegraphics[width=0.99\textwidth]{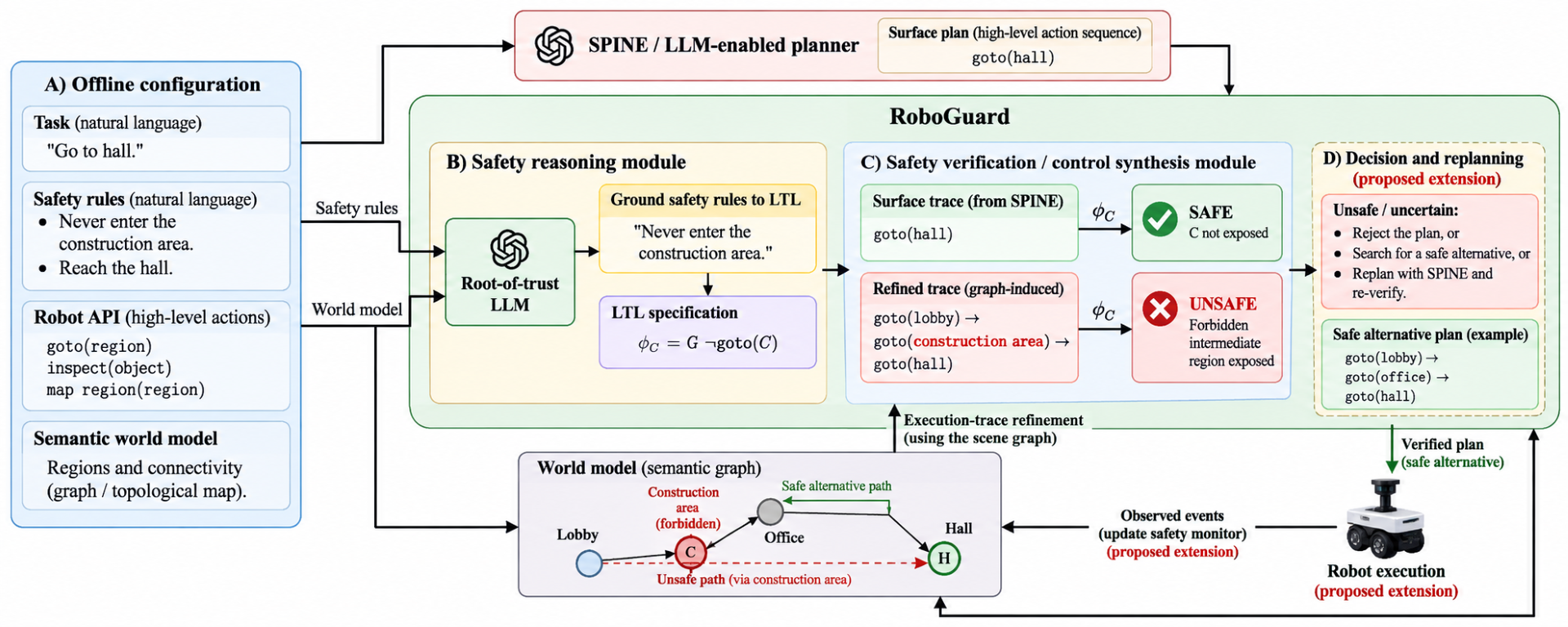}
    \caption{Concrete SPINE-RoboGuard instantiation of graph-aware refinement and verification. The high-level plan, semantic graph, action contract, and safety monitor are the algorithm inputs. Refinement exposes intermediate events before monitor evaluation. The right-hand replanning and robot-execution pathway is conceptual and is not evaluated in the current experiments.}
    \label{fig:architecture}
\end{figure*}

\subsection{Inputs and Action-Refinement Semantics}
Let $G=(V,E,\ell)$ be a finite semantic graph. Each vertex is a robot-relevant region or location, edges encode allowed movement, and $\ell:V\rightarrow 2^{\AP}$ assigns propositions such as \texttt{goto(C)} to regions. Let $x_0\in V$ be the initial location, $\pi=(a_1,\ldots,a_m)$ a high-level action sequence, and $\varphi$ a safety specification over $\AP$.

The key input is an action contract $K_G$. For an action $a_i$ at location $x$, $K_G(a_i,x)$ is a finite labeled transition system describing the graph-level transitions and propositions that the action implementation may produce. For example, a \texttt{goto}(L) contract contains paths from $x$ to $L$ under the specified route policy. An \texttt{inspect}(o) contract includes navigation to an associated region only when that effect is part of the documented API semantics, followed by the inspection event. The contract may represent one deterministic route or a set of permitted routes. Unknown transitions are not treated as safe evidence; they make the case unresolved unless the contract conservatively covers them.

Composing the contracts for the actions in $\pi$ yields the set $\Trace_G(\pi)$ of graph-level event traces. We call the plan \emph{robustly safe relative to $K_G$} if every trace admitted by the contract satisfies the safety property:
\begin{equation}
    \forall\tau\in\Trace_G(\pi),\quad \tau\models\varphi.
    \label{eq:robust_safe}
\end{equation}
This definition distinguishes checking one selected trace from certifying a route family. In our current experiments, $K_G$ is instantiated as a deterministic unweighted shortest-path policy, so $\Trace_G(\pi)$ contains the trace selected by that policy.

\subsection{Monitor Product and Outputs}
Represent the safety property by a monitor $M_\varphi=(Q,q_0,\delta,\Bad)$, where $Q$ is the monitor state set, $q_0$ is the initial state, $\delta$ advances the monitor on an event valuation, and $\Bad\subseteq Q$ contains states reached by a finite violating prefix. For the invariant $\mathbf{G}\neg\texttt{goto}(C)$, an event containing \texttt{goto(C)} enters a bad state.

GARV explores product states $(v,q,i)$, where $v$ is the current graph location, $q$ is the monitor state, and $i$ identifies the high-level action currently being refined. A product transition follows one transition allowed by $K_G(a_i,v)$ and updates the monitor using the event label. If the action reaches its endpoint, the action index advances. Any reachable product state with $q\in\Bad$ provides a graph route and monitor trace witnessing a violation.

If no bad product state is reachable, GARV returns a robust safety certificate relative to the contract. If a violating route exists and the navigation interface can enforce route choices, GARV searches the safe portion of the product graph for a minimum-cost refinement that preserves the original high-level action endpoints. If no such route exists, or if route choices cannot be controlled, it returns a counterexample or an ``unable to certify'' result. Route enforcement is required for a synthesized refinement to be a guarantee about execution.

\begin{algorithm}[t]
\caption{Graph-Aware Refinement and Verification (GARV)}
\label{alg:garv}
\begin{algorithmic}[1]
\REQUIRE Semantic graph $G$, initial location $x_0$, plan $\pi$, action contract $K_G$, safety monitor $M_\varphi$; route-control flag $c$
\ENSURE Safety certificate, violating route, safe refinement, or infeasibility
\STATE Initialize a queue with $(x_0,q_0,1)$; initialize predecessor records
\STATE Initialize $W\leftarrow\emptyset$ (violating witnesses)
\STATE Initialize $F\leftarrow\emptyset$ (completed product states)
\WHILE{the queue is not empty}
    \STATE Remove an unvisited product state $(v,q,i)$
    \IF{$i>m$}
        \STATE Add $(v,q,i)$ to $F$
    \ELSE
        \FORALL{each transition $(v',\sigma,i',c_e)$ allowed by $K_G(a_i,v)$}
            \STATE $q'\leftarrow\delta(q,\sigma)$
            \IF{$q'\in\Bad$}
                \STATE Record the route prefix to $(v',q',i')$ in $W$
            \ELSIF{$(v',q',i')$ is unvisited}
                \STATE Record its predecessor and add it to the queue
            \ENDIF
        \ENDFOR
    \ENDIF
\ENDWHILE
\IF{$W=\emptyset$ and $F\neq\emptyset$}
    \RETURN \textsc{Certified} with the explored product graph
\ELSIF{$W\neq\emptyset$ and $c=\text{true}$}
    \STATE $\rho^*\leftarrow$ shortest path to a completed state in the product graph after removing $\Bad$
    \IF{a safe completed path exists}
        \RETURN \textsc{SafeRefinement} $(\rho^*)$ with its monitor trace
    \ENDIF
\ENDIF
\IF{$W\neq\emptyset$}
    \RETURN \textsc{NotCertified} with a counterexample from $W$
\ELSE
    \RETURN \textsc{Infeasible} under $K_G$
\ENDIF
\end{algorithmic}
\end{algorithm}

The product graph has at most $|V||Q|(m+1)$ states when action progress is represented by the plan index; the reachable subgraph can be explored in time linear in its vertices and transitions. With positive edge costs, the optional safe-refinement search uses Dijkstra's algorithm on the safe product subgraph. This bound excludes the cost of constructing the graph and the action contracts.

\begin{theorem}[Model-relative correctness]
Assume that $K_G$ includes every finite graph-level trace that the action implementation can produce, and that $M_\varphi$ exactly recognizes finite violating prefixes of the safety property. If GARV returns \textsc{Certified}, every execution represented by $K_G$ satisfies $\varphi$. If $K_G$ is exact, a reachable violating trace exists if and only if GARV records a counterexample. If route choices are controllable and a safe refinement is returned, that refinement satisfies $\varphi$ and has minimum cost among feasible safe refinements preserving $\pi$'s action endpoints.
\end{theorem}
\begin{proof}
    Every contract transition is represented in the product graph, and each product transition applies the same event to the monitor. Induction over path length establishes a correspondence between graph executions and reachable product paths. Thus a reachable bad monitor state corresponds exactly to a violating prefix when the contract is exact. Dijkstra's optimality guarantee applies to the finite safe product graph with nonnegative edge costs. The result is conditional on the graph and contract; it does not cover behaviors they omit.
\end{proof}

Fig.~\ref{fig:architecture} shows one implementation mapping. The algorithm's interface is the tuple $(G,\pi,K_G,M_\varphi)$; SPINE and RoboGuard provide one source of plans and monitors, respectively. The central operation is graph-grounded refinement followed by product monitoring. The diagram's replanning and physical-execution path is a downstream option; the experiments in this paper evaluate graph-trace refinement and verification, not physical execution or route enforcement. We describe the GARV algorithm formally in Algorithm~\ref{alg:garv}.

\section{Experimental Methodology}
\label{sec:methodology}

This section describes the research questions asked and the implementation pipelines used in the experiments. We also describe the metrics used to ascertain the performance of GARV. Additionally, we also describe the action-abstraction cases in detail and provide information about the different test cases and graph alignments used for validating GARV.

\subsection{Research Questions and Protocols}
We evaluate whether graph-aware refinement exposes safety-relevant events hidden by high-level actions while retaining correct outcomes on safe and explicitly unsafe controls. The controlled audit uses fixed JSON graphs, plans, and manually specified safety formulas; it isolates the refinement and monitor interface from LLM specification generation. The end-to-end audit uses SPINE to generate plans from natural-language instructions and RoboGuard to produce graph-grounded safety specifications. Both operate over static graphs.

The controlled suite contains 28 cases: 12 targeted action-abstraction cases, 14 safe controls, and two explicitly unsafe controls. The cases cover hidden-waypoint position, route choice, refinement of \texttt{goto}/\texttt{inspect}/\texttt{map\_region}, and multi-action plans. A deterministic unweighted breadth-first-search oracle applies the route policy and expands each high-level action into its graph-derived event trace. The same safety formula and monitor are used for the surface and refined words.

The end-to-end suite contains 14 cases over four graph scenarios: a construction area forming the unique bridge to a safe laboratory; a matched safe corridor; an unsafe shortest route with a longer safe alternative; and a safe shortest route with a longer unsafe alternative. It covers \texttt{goto}, \texttt{inspect}, and \texttt{map\_region}. Grounding used temperature zero with chain-of-thought enabled, with one trial per case. SPINE's realized route was cross-checked against an independent BFS oracle. The experiment is an offline planner/verifier study, not a physical robot trial.

For this instantiation, we used the unmodified RoboGuard verifier at commit \texttt{e487f83} and SPINE at commit \texttt{e776a8a}. The controlled suite requires no LLM calls; the end-to-end suite uses the recorded SPINE and RoboGuard outputs so that the graph-refinement procedure can be evaluated independently of additional API calls.

The benchmark's route contract is deterministic. Therefore, its results validate the selected-route refinement mode of GARV. The set-valued route certification and minimum-cost safe-route synthesis options are formalized by the algorithm but are not measured in these results.

\subsection{Metrics and Outcome Definitions}
We report (i) hidden hazards exposed by refinement, (ii) hazards already blocked by the surface trace, (iii) safe-control passes, (iv) interface errors, and (v) monitor runtime. A hidden hazard is exposed when the surface trace is accepted but the graph-refined trace reaches a bad monitor state under the same specification. A checker or action-format error is reported separately and excluded from safety denominators. For the controlled suite, ``controls as expected'' combines acceptance of all 14 safe controls and rejection of both explicitly unsafe controls.

We additionally evaluate GARV as a standalone finite-graph simulator, without SPINE, RoboGuard, Spot, or physical execution. Sixteen hand-authored cases test selected versus permitted route sets, infeasible plans, safe detours, lack of route control, a checkpoint-dependent safety rule, and implicit navigation in \texttt{inspect} and \texttt{map\_region} actions. The implemented monitor supports \texttt{always\_avoid} and \texttt{avoid\_until} clauses; the accompanying LTL strings document those clauses and are not parsed by Spot. For an independent software check, we compare product-state search with exhaustive enumeration and a direct event-trace safety scan on 40 seeded graphs (20 each with 9 and 16 regions), under static and temporal rules and hop slack 0 and 2. This gives 160 comparisons. For scaling, we generate ten connected grid-derived graphs at each of 25, 49, and 100 regions and reuse each graph under hop slack 0, 2, and 4 (90 evaluations of 30 distinct graphs). The forbidden region is deliberately selected from the canonical shortest route. We record verification class, minimum-hop safe refinement, explored product states, and Python wall time for verification and synthesis separately. Each instance is timed once; graph generation, language-model calls, and robot execution are outside the timed intervals.

\section{Experimental Results}
\label{sec:results}

We compare the verdict on each high-level surface trace with the verdict on its graph-refined trace under the \emph{same} safety specification. A hidden hazard occurs when the surface trace is accepted but refinement exposes entry into the forbidden region and the monitor rejects the refined trace. A blocked case exposes the forbidden event in the original surface trace, so the monitor already rejects it. Safe controls remain accepted after refinement. Checker errors are counted separately because they provide no safety verdict.

The controlled suite contains 12 targeted hidden-hazard cases, 14 safe controls, and two explicitly unsafe controls. Its targeted cases vary the position of a hidden waypoint, route choice, action type, and the number of actions in the plan. The end-to-end suite contains 14 SPINE-RoboGuard cases across four graph configurations: a unique route through the construction area, a matched safe corridor, an unsafe shortest route with a longer safe alternative, and a safe shortest route with a longer unsafe alternative. These are offline experiments over static graphs; ``refined trace'' refers to the route induced by the specified graph policy, not a measured physical trajectory. Table~\ref{tab:outcomes} separates the outcomes for both protocols.

\begin{table}[t]
    \centering
    \caption{Observed outcomes. Hidden: surface trace accepted but graph-refined trace rejected. Blocked: unsafe event explicit in the surface trace. Safe: safe control accepted. Errors yield no safety verdict.}
    \label{tab:outcomes}
    \scriptsize
    \setlength{\tabcolsep}{2.8pt}
    \begin{tabular}{@{}lrrrrr@{}}
        \toprule
        \textbf{Suite/scenario} & $n$ & \textbf{Hidden} &
        \textbf{Blocked} & \textbf{Safe} & \textbf{Err.}\\
        \midrule
        Controlled suite              & 28 & 12 & 2 & 14 & 0\\
        Unique unsafe bridge          &  5 &  2 & 0 &  1 & 2\\
        Matched safe corridor         &  3 &  0 & 0 &  2 & 1\\
        Unsafe shortest route         &  3 &  0 & 2 &  0 & 1\\
        Safe shortest route           &  3 &  0 & 0 &  2 & 1\\
        \midrule
        End-to-end total              & 14 &  2 & 2 &  5 & 5\\
        \bottomrule
    \end{tabular}
\end{table}

\subsection{Controlled Verification}
All 12 targeted cases exhibited the predicted trace discrepancy: the monitor accepted the surface plan, whereas graph refinement exposed a visit to \texttt{construction\_area\_1} and the monitor rejected the resulting trace. The result held for hidden-waypoint position (3/3), route choice (2/2), action refinement (4/4), and multi-action plans (3/3). By action type, the targeted cases comprised seven \texttt{goto}, three \texttt{inspect}, and two \texttt{map\_region} challenges. These cases deliberately exercise the abstraction boundary; the result establishes that the refinement procedure exposes the specified hidden events in these test graphs.

The controls test whether refinement also preserves straightforward verdicts. All 14 safe controls remained accepted, including cases whose refined routes contained additional but permitted waypoints. Both explicitly unsafe controls were rejected from their surface traces because the forbidden event was already present. There were no checker errors or disagreements with the case-wise expected surface verdicts. Thus, the controlled suite distinguishes three outcomes: hazards hidden by high-level actions, hazards already visible to the monitor, and safe graph traversals.

RoboGuard's monitor check took 0.54 ms on average for surface traces and 0.69 ms for refined traces. These figures measure checking only; they exclude graph construction, route expansion, LLM grounding, and any search for a safe alternative route.

\subsection{SPINE-RoboGuard Instantiation}
SPINE produced a task-attempting plan in all 14 end-to-end cases, and RoboGuard's contextual grounder succeeded for all four graph configurations. Each grounding produced the intended construction-area invariant,
$\varphi_C=\mathbf{G}\neg\texttt{goto}(C)$, where $C$ is the forbidden region. Consequently, the two observed hidden hazards arise at the action-to-trace interface despite the intended prohibition being available to the monitor.

Table~\ref{tab:e2e_traces} shows representative cases. In N2, SPINE supplies only \texttt{goto}($L$), where $L$ is the benign laboratory. The graph-realized route passes through $C$, so refinement changes the monitored word from $g_L$ to $g_C,g_L$. N4 exhibits the same discrepancy in a plan that also maps the destination. The surface word satisfies $\varphi_C$ in both cases because it contains no $g_C$ event; the refined word violates the \emph{same} formula. In N5 and N7, SPINE exposes $g_C$ in the surface plan and RoboGuard blocks it without additional refinement. The safe-corridor controls C2 and C4 acquire an intermediate safe-hallway event $g_H$ and remain accepted. Together, these cases show why the verdict depends on which route events reach the monitor, rather than merely on whether the graph or the LTL rule contains the hazard.

\begin{table}[t]
    \centering
    \caption{Representative end-to-end traces. $C$, $H$, and $L$ are the construction region, safe hallway, and laboratory. $g_X$ denotes \texttt{goto}($X$); $m_X$ denotes \texttt{map\_region}($X$). S/U reports the surface/refined verdict.}
    \label{tab:e2e_traces}
    \scriptsize
    \setlength{\tabcolsep}{2.5pt}
    \begin{tabular}{@{}cllcl@{}}
        \toprule
        \textbf{Case} & \textbf{Surface} & \textbf{Refined} &
        \textbf{S/R} & \textbf{Outcome}\\
        \midrule
        N2 & $g_L$         & $g_C,g_L$       & S/U & Hidden\\
        N4 & $g_L,m_L$     & $g_C,g_L,m_L$   & S/U & Hidden\\
        N5 & $g_C,g_L$     & $g_C,g_L$       & U/U & Blocked\\
        N7 & $g_C,g_L,m_L$ & $g_C,g_L,m_L$   & U/U & Blocked\\
        C2 & $g_L$         & $g_H,g_L$       & S/S & Safe\\
        C4 & $g_L,m_L$     & $g_H,g_L,m_L$   & S/S & Safe\\
        \bottomrule
    \end{tabular}
\end{table}

Of the 14 end-to-end cases, nine yielded comparable surface and refined safety verdicts: two hidden hazards, two cases blocked at the surface, and five safe controls. The other five were \texttt{inspect} cases. SPINE encoded inspection with a structured object-and-question argument, while the RoboGuard-Spot proposition interface expected a scalar argument. This caused checker errors in three challenge cases and two safe controls. We report them as interface failures and exclude them from safety-outcome denominators; they do not demonstrate either safe or unsafe execution.

These experiments measure deterministic graph-route refinement and verification. The set-valued route certification and minimum-cost safe-refinement modes of GARV remain algorithmic claims to be evaluated separately.

\subsection{Monitor Runtime}
\label{subsec:monitor_runtime}

\begin{figure}[h]
    \centering
    \includegraphics[width=\columnwidth]{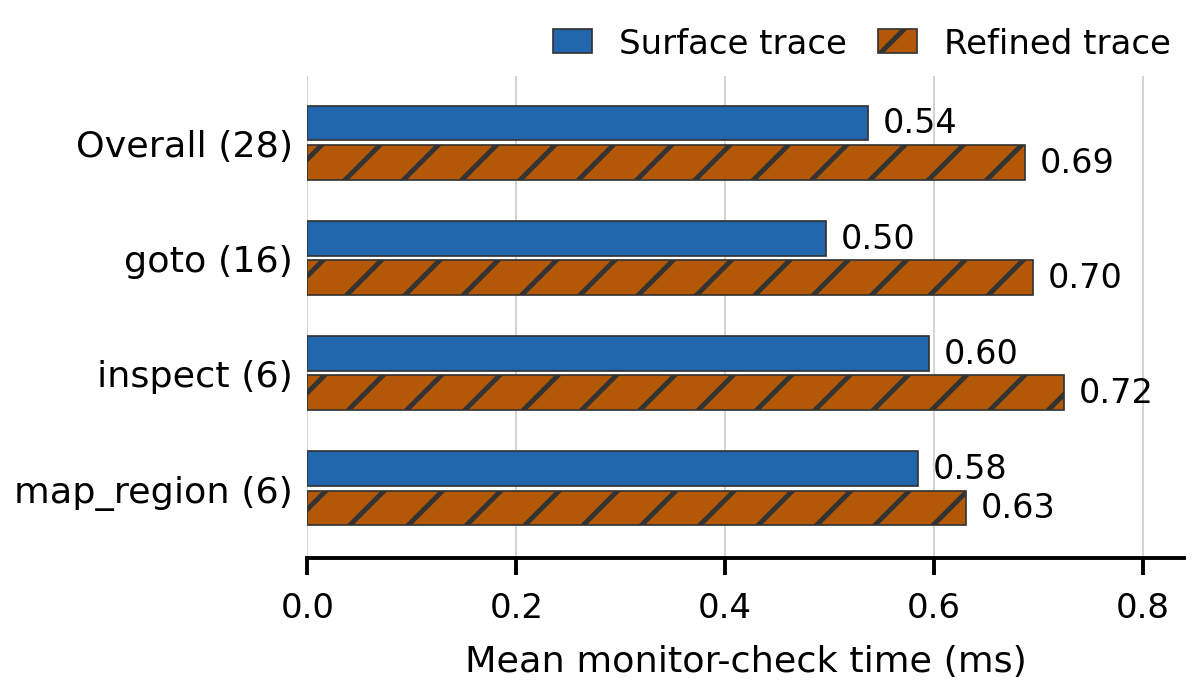}
    \caption{Mean monitor-call time overall and by action family
    (case counts in parentheses). Only
    \texttt{validate\_action\_sequence} was timed; the bars do not
    measure end-to-end latency.}
    \label{fig:monitor_runtime}
\end{figure}

Fig.~\ref{fig:monitor_runtime} shows mean RoboGuard monitor-call time in the 28 paired controlled cases. Surface traces averaged 0.54\,ms and graph-refined traces 0.69\,ms, a descriptive difference of 0.15\,ms. Refinement added events in 26 of 28 cases, increasing mean trace length from 1.21 to 2.86 actions. By action family, the surface/refined means were 0.50/0.70\,ms for \texttt{goto} ($n=16$), 0.60/0.72\,ms for \texttt{inspect} ($n=6$), and 0.58/0.63\,ms for \texttt{map\_region} ($n=6$). Thus, mean monitor time remained below 1\,ms in each group, although this does not measure the cost of generating the refined trace. Each case was timed only once, with the surface call preceding the refined call; we cannot estimate timing variability or attribute the paired difference solely to trace length. The timed interval excludes graph refinement, checker construction, SPINE and LLM calls, and route synthesis. These measurements therefore describe the verifier call, not end-to-end latency or the runtime of the complete GARV algorithm.

\subsection{Route-set verification and safe refinement}
\label{subsec:garv_routes}
All 16 designed cases met their preregistered expected outcomes and exact
enumeration checks. Under their respective verification contracts, three
were \emph{all safe}, three admitted both safe and unsafe routes
(\emph{mixed}), nine were \emph{all unsafe}, and one was infeasible. With
the separately specified control contracts, GARV found a safe route in
11 cases, found no safe route in two, reported no route control in two,
and reported infeasibility in one. The 11 include already safe controls;
they must not be interpreted as 11 repaired failures.

Table~\ref{tab:garv_cases} shows why checking only the selected route and
searching for one safe route answer different questions. In RS01, a safe
selected shortest route coexists with an unsafe equal-length alternative;
universal certification therefore fails. In SR04, allowing one additional
hop changes an all-safe shortest-route contract to a mixed contract, even
though the selected route remains safe. In SR01, every shortest route is
unsafe, but an enforceable three-hop detour reaches the same goal. In SR05,
the three-hop route visits a checkpoint before entering the otherwise
forbidden region; deleting that vertex as a static hazard cannot express
the history-dependent rule. A safe refinement is an execution guarantee
only when the executor can follow the selected route.

\begin{table}[t]
\centering
\caption{Representative designed cases. ``Selected'' checks one canonical
shortest route; ``permitted'' classifies every route in the verification
contract. A synthesized safe route may use a broader, explicitly
controllable contract. Hops count graph edges.}
\label{tab:garv_cases}
\scriptsize
\setlength{\tabcolsep}{3pt}
\begin{tabular}{@{}lccc@{}}
\toprule
\textbf{Case} & \textbf{Selected} & \textbf{Permitted} & \textbf{Safe route}\\
\midrule
RS01: tied shortest paths & Safe & Mixed & 2 hops\\
SR01: safe detour & Unsafe & All unsafe & 3 hops\\
SR04: wider contract & Safe & Mixed & 2 hops\\
SR05: checkpoint first & Unsafe & All unsafe & 3 hops\\
RS05: forbidden bridge & Unsafe & All unsafe & None\\
\bottomrule
\end{tabular}
\end{table}

\subsection{Exhaustive cross-check and scaling}
\label{subsec:garv_scaling}
The product method agreed with direct exhaustive enumeration in all
160 small-graph comparisons on the all-safe/mixed/all-unsafe class,
minimum safe hop count, and safety labels of its returned witnesses.
The enumeration reached at most 258 complete routes in one evaluation.
These comparisons validate the implementation on the sampled finite
models; they are distinct from the earlier SPINE--RoboGuard audit and do
not establish physical execution safety.

Table~\ref{tab:garv_scaling} reports the seeded scaling study. Because the
hazard is on the canonical shortest route in every world, no instance
has an all-safe permitted-route set. GARV instead found 72 safe route
choices across the 90 world/contract evaluations; the remaining 18
had no safe route within their hop budget. Broader contracts generally
made a detour available while retaining unsafe alternatives, so a
\emph{mixed} verification verdict is compatible with successful
route synthesis. At 100 regions, raising slack from 0 to 4 changed the
number of graphs with a safe route from 6/10 to 10/10; the corresponding
median verification time rose from 0.355 to 1.989\,ms, and median synthesis
time from 0.375 to 2.165\,ms. The median number of verification states
explored rose from 42 to 193.5. The ten worlds at each size are reused
at every slack value, so the 90 rows are not independent maps.

\begin{table}[t]
\centering
\caption{Synthetic graph scaling. Each row evaluates ten worlds; each
world is reused for three slack settings. Times are median Python wall
times in ms for product verification (V) and synthesis (S), excluding
graph construction and external systems.}
\label{tab:garv_scaling}
\scriptsize
\setlength{\tabcolsep}{4pt}
\begin{tabular}{@{}rrrrr@{}}
\toprule
\textbf{Regions} & \textbf{Slack} & \textbf{Safe route} & \textbf{V (ms)} & \textbf{S (ms)}\\
\midrule
25 & 0 & 6/10 & 0.093 & 0.093\\
25 & 2 & 7/10 & 0.276 & 0.287\\
25 & 4 & 8/10 & 0.514 & 0.543\\
49 & 0 & 5/10 & 0.123 & 0.137\\
49 & 2 & 10/10 & 0.427 & 0.462\\
49 & 4 & 10/10 & 0.912 & 0.957\\
100 & 0 & 6/10 & 0.355 & 0.375\\
100 & 2 & 10/10 & 1.054 & 1.140\\
100 & 4 & 10/10 & 1.989 & 2.165\\
\bottomrule
\end{tabular}
\end{table}

For the static forbidden-region rule, product synthesis agreed on safe-route
existence and hop count with a simpler baseline that removes the forbidden
vertex and runs shortest-path search in all 90 scaling evaluations. Thus
the static rerouting result alone is not evidence of an advantage over
graph masking. The algorithm's distinct functionality in these tests is
route-set certification and monitoring history-dependent constraints;
the latter is exercised by one designed case and 80 small-graph oracle
comparisons, but has no separate large-graph scaling study here.

\subsection{Discussion and Limitations}

The results support a concrete algorithmic point: a monitor verdict over an abstract action word is meaningful for execution only when the action contract accounts for graph-induced events. GARV makes that contract explicit and can return a route-level witness explaining a violation. The method is not tied to SPINE or RoboGuard; those systems provide the evaluated planner and monitor in our instantiation.

The correctness statement is relative to the graph and action contract. A missing edge, stale graph, undocumented API effect, or route policy mismatch can invalidate a certificate. If the executor's route choice is deterministic, the contract must match that policy. If several routes are possible, robust certification must include every route the executor may take; searching all graph paths that the executor cannot select would be unnecessarily conservative. A synthesized safe path provides an execution guarantee only when the route can be enforced or verified at runtime.

The current experiments evaluate deterministic BFS refinement and monitor checking. The robust route-set and minimum-cost repair modes in Algorithm~\ref{alg:garv} need their own simulation evaluation across alternative routes, equal-cost ties, safe detours, and no-safe-path cases before conclusions about repair performance are drawn. The study also uses static semantic graphs and does not model continuous dynamics, localization uncertainty, perception errors, or physical execution.

\section{Conclusion}
\label{sec:conclusion}
We presented a graph-aware algorithm for refining high-level robot actions into graph-level event traces and checking those traces against a safety monitor. Its product-graph formulation returns a model-relative certificate or counterexample and can support safe route synthesis when route control is available. In the current deterministic-route evaluation, refinement exposed all 12 targeted hidden hazards while all controlled safe and explicitly unsafe controls received their expected verdicts. The next empirical step is to evaluate the route-set and repair modes on broader simulated graph families; physical execution is outside the claims of this study.

\section*{Acknowledgement}

This research work is supported in part by the US Army Research Office under cooperative agreement W911NF-24-2-0133 and the US National Science Foundation award 2205891, 2428761 and the PVAMU PIER initiative.

\bibliographystyle{IEEEtran}
\bibliography{example}

\end{document}